\documentclass{article}
\usepackage{ijcai22}

\usepackage{times}
\usepackage{soul}
\usepackage{url}
\usepackage[hidelinks]{hyperref}
\usepackage[utf8]{inputenc}
\usepackage[small]{caption}
\usepackage{graphicx}
\usepackage{amsmath}
\usepackage{amsthm}
\usepackage{booktabs}
\usepackage{algorithm}
\usepackage{algorithmic}
\usepackage{dsfont}

\usepackage{float}
\usepackage{amssymb}
\usepackage{array}
\usepackage[mathscr]{eucal}
\usepackage{multirow}
\usepackage{color}
\usepackage{bbm}
\usepackage{tikz}
\usetikzlibrary{automata,arrows,positioning}
\definecolor {processblue}{cmyk}{0.96,0,0,0}
\usepackage{cleveref}
\usepackage[autostyle]{csquotes}
\usepackage{mathtools}
\DeclarePairedDelimiterX{\norm}[1]{\lVert}{\rVert}{#1}

\newtheorem{theorem}{Theorem}
\newtheorem{remark}{Remark}
\title{Linear Combinatorial Semi-Bandit with Causally Related Rewards}
\title{Linear Combinatorial Semi-Bandit with Causally Related Rewards}

\author{
Behzad Nourani-Koliji$^{1}$\footnote{Equal Contribution}\and
Saeed Ghoorchian$^{2}$$^*$\And 
Setareh Maghsudi$^{1}$\\
\affiliations
$^1$University of Tübingen\\
$^2$Technical University of Berlin\\
\emails
behzad.nourani-koliji@uni-tuebingen.de
saeed.ghoorchian@tu-berlin.de  setareh.maghsudi@uni-tuebingen.de}
\begin{document}

\maketitle
%-----------------------------------------> Abstract
\begin{abstract}
  In a sequential decision-making problem, having a structural dependency amongst the reward distributions associated with the arms makes it challenging to identify a subset of alternatives that guarantees the optimal collective outcome. Thus, besides individual actions' reward, learning the causal relations is essential to improve the decision-making strategy. To solve the two-fold learning problem described above, we develop the 'combinatorial semi-bandit framework with causally related rewards', where we model the causal relations by a directed graph in a stationary structural equation model. The nodal observation in the graph signal comprises the corresponding base arm's instantaneous reward and an additional term resulting from the causal influences of other base arms' rewards. The objective is to maximize the long-term average payoff, which is a linear function of the base arms' rewards and depends strongly on the network topology. To achieve this objective, we propose a policy that determines the causal relations by learning the network's topology and simultaneously exploits this knowledge to optimize the decision-making process. We establish a sublinear regret bound for the proposed algorithm. Numerical experiments using synthetic and real-world datasets demonstrate the superior performance of our proposed method compared to several benchmarks.
\end{abstract}
%-----------------------------------------> Introduction
\section{Introduction}
\label{sec:Intro}
In the seminal form of the Multi-Armed Bandit (MAB) problem, an agent selects an arm from a given set of arms at sequential rounds of decision-making. Upon selecting an arm, the agent receives a reward, which is drawn from the unknown reward distribution of that arm. The agent aims at maximizing the average reward over the gambling horizon \cite{robbins1952some}. The MAB problem portrays the exploration-exploitation dilemma, where the agent decides between accumulating immediate reward and obtaining information that might result in larger reward only in the future \cite{Maghsudi16:IWC}. To measure the performance of a strategy, one uses the notion of \textit{regret}. It is the difference between the accumulated reward of the applied decision-making policy and that of the optimal policy in hindsight.

In a combinatorial semi-bandit setting \cite{chen2013combinatorial}, at each round, the agent selects a subset of \textit{base arms}. This subset is referred to as a \textit{super arm}. She then observes the individual reward of each base arm that belongs to the selected super arm. Consequently, she accumulates the collective reward associated with the chosen super arm. The combinatorial MAB problem is challenging since the number of super arms is combinatorial in the number of base arms. Thus, conventional MAB algorithms such as \cite{auer2002finite} are not appropriate for combinatorial problems as they result in suboptimal regret bounds. The aforementioned problem becomes significantly more difficult when there are causal dependencies amongst the reward distributions.

In some cases, it is possible to model the causal structure that affects the rewards \cite{lattimore2016causal}. Therefore, exploiting the knowledge of this structure helps to deal with the aforementioned challenges. In our paper, we develop a novel combinatorial semi-bandit framework with causally related rewards, where we rely on Structural Equation Models (SEMs) \cite{kaplan2008structural} to model the causal relations. At each time of play, we see the instantaneous rewards of the chosen base arms as controlled stimulus to the causal system. Consequently, in our causal system, the solution to the decision-making problem is the choice over the exogenous input that maximizes the collected reward. We propose a decision-making policy to solve the aforementioned problem and prove that it achieves a sublinear regret bound in time. Our developed framework can be used to model various real-world problems, such as network data analysis of biological networks or financial markets. We apply our framework to analyze the development of Covid-19 in Italy. We show that our proposed policy is able to detect the regions that contribute the most to the spread of Covid-19 in the country.

Compared to previous works, our proposed framework does not require any prior knowledge of the structural dependencies. For example, in \cite{tang2017networked}, the authors exploit the prior knowledge of statistical structures to learn the best combinatorial strategy. At each decision-making round, the agent receives the reward of the selected super arm and some side rewards from the selected base arms' neighbors. In \cite{huyuk2019analysis} a Combinatorial Thompson Sampling (CTS) algorithm to solve a combinatorial semi-bandit problem with probabilistically triggered arms is proposed. The proposed algorithm has access to an oracle that determines the best decision at each round of play based on the already collected data. Similarly, the authors in \cite{chen2016combinatorial} study a setting where triggering super arms can probabilistically trigger other unchosen arms. They propose an Upper Confidence Bound (UCB)-based algorithm that uses an oracle to improve the decision-making process. In \cite{yu2020graphical}, the authors formulate a combinatorial bandit problem where the agent has access to an influence diagram that represents the probabilistic dependencies in the system. The authors propose a Thompson sampling algorithm and its approximations to solve the formulated problem. Further, there are some works that study the underlying structure of the problem. For example, in \cite{toni2018spectral}, the authors attempt to learn the structure of a combinatorial bandit problem. However, they do not assume any causal relations between rewards. Moreover, in \cite{sen2017identifying}, the MAB framework is employed to identify the best soft intervention on a causal system while it is assumed that the causal graph is only partially unknown.

The rest of the paper is organized as follows. In Section \ref{sec:ProFor}, we formulate the structured combinatorial semi-bandit problem with causally related rewards. In Section \ref{sec:DMStrategy}, we introduce our proposed algorithm, namely SEM-UCB. Section \ref{sec:RegAnalysis} includes the theoretical analysis of the regret performance of SEM-UCB. Section \ref{sec:NumAnalysis} is dedicated to numerical evaluation. Section \ref{sec:Conclusion} concludes the paper.

%-----------------------------------------> Problem Formulation
\section{Problem Formulation}
\label{sec:ProFor}
Let $[N] = \{1, 2, \dots, N\}$ denote the set of \textit{base arms}. $\mathbf{b}_{t} = [\mathbf{b}_{t}[1], \mathbf{b}_{t}[2], \dots, \mathbf{b}_{t}[N]] \in [0,1]^{N}$ represents the vector of \textit{instantaneous rewards} of the base arms at time $t$. The instantaneous rewards of each base arm $i \in [N]$ are independent and identically distributed (i.i.d.) random variables drawn from an unknown probability distribution with mean $\boldsymbol{\beta}[i]$. We collect the mean rewards of all the base arms in the mean reward vector of $\boldsymbol{\beta} = [\boldsymbol{\beta}[1], \boldsymbol{\beta}[2], \dots, \boldsymbol{\beta}[N]]$.

We consider a causally structured combinatorial semi-bandit problem where an agent sequentially selects a subset of base arms over time. We refer to this subset as the \textit{super arm}. More precisely, at each time $t$, the agent selects a \textit{decision vector} $\mathbf{x}_{t} = [\mathbf{x}_{t}[1], \mathbf{x}_{t}[2], \dots, \mathbf{x}_{t}[N]] \in \left \{ 0,1 \right \}^{N}$. If the agent selects the base arm $i$ at time $t$, we have $\mathbf{x}_{t}[i] = 1$, otherwise $\mathbf{x}_{t}[i] = 0$. The agent observes the value of $\mathbf{b}_{t}[i]$ at time $t$ only if $\mathbf{x}_{t}[i] = 1$. The agent is allowed to select at most $s$ base arms at each time of play. Hence, we define the set of all feasible super arms as
\begin{equation}
\label{eq:SuperArmSet}
    \mathcal{X} = \left \{ \mathbf{x} \mid \mathbf{x} \in \{0,1\}^{N} \wedge \left \| \mathbf{x} \right \|_{0}  \leq s \right\},
\end{equation}
where $\left\| \cdot \right \|_{0}$ determines the number of non-zero elements in a vector. In our problem, the parameter $s$ is pre-determined and is given to the agent.

We take advantage of a directed graph structure to model the causal relationships in the system. We consider an unknown stationary sparse Directed Acyclic Graph (DAG) $\mathcal{G} = (\mathcal{V}, \mathcal{E}, \mathbf{A})$, where $\mathcal{V}$ denotes the set of $N$ vertices, i.e., $\left | \mathcal{V} \right | = N$, $\mathcal{E}$ is the edge set, and $\mathbf{A}$ denotes the weighted adjacency matrix. By $p \leq N - 1$, we denote the length of the largest path in the graph $\mathcal{G}$. We assume that the reward generating processes in the bandit setting follow an error-free Structural Equation Model (SEM) (\cite{giannakis2018topology}, \cite{bazerque2013identifiability}). The exogenous input vector and the endogenous output vector of the SEM at each time $t$ are denoted by $\mathbf{z}_{t} = [\mathbf{z}_{t}[1], \mathbf{z}_{t}[2], \dots, \mathbf{z}_{t}[N]]$ and $\mathbf{y}_{t} = [\mathbf{y}_{t}[1], \mathbf{y}_{t}[2], \dots, \mathbf{y}_{t}[N]]$, respectively. At each time $t$, the exogenous input $\mathbf{z}_{t}$ represents the semi-bandit feedback in the decision-making problem. Formally,
\begin{align}
\label{eq:ExoInput}
    \mathbf{z}_{t} = \operatorname{diag}(\mathbf{b}_{t}) \mathbf{x}_{t},
\end{align}
where $\operatorname{diag}(\cdot)$ represents the diagonalization of its given input vector. Consequently, we define the elements of the endogenous output vector $\mathbf{y}_{t}$ as
\begin{equation}
\label{eq:OverallRew}
\mathbf{y}_{t}[i] = \sum_{i \neq j} \mathbf{A}[i,j] \mathbf{y}_{t}[j] + \mathbf{F}[i,i] \mathbf{z}_{t}[i],~~~~\forall i = 1, \dots, N,
\end{equation}
%According to our system model, the matrix $\mathbf{F}$ is known.
where $\mathbf{F}$ is a diagonal matrix that captures the effects of the exogenous input vector $\mathbf{z}_{t}$. The SEM in \Cref{eq:OverallRew} implies that the output measurement $\mathbf{y}_{t}[i]$ depends on the single-hop neighbor measurements in addition to the exogenous input signal $\mathbf{z}_{t}[i]$. In our formulation, at each time $t$, the endogenous output $\mathbf{y}_{t}[i]$ represents the \textit{overall reward} of the corresponding base arm $i \in [N]$. Therefore, at each time $t$, the overall reward of each base arm comprises two parts; one part directly results from its instantaneous reward, while the other part reflects the effect of causal influences of other base arms' overall rewards.

In \Cref{eq:OverallRew}, the overall rewards are causally related. Thus, the adjacency matrix $\mathbf{A}$ represents the causal relationships between the overall rewards; accordingly, the element $\mathbf{A}[i,j]$ of the adjacency matrix $\mathbf{A}$ denotes the causal impact of the overall reward of base arm $j$ on the overall reward of base arm $i$, and we have $\mathbf{A}[i,i] = 0$, $\forall i = 1, 2, \dots, N$. We assume that the agent is not aware of the causal relationships between the overall rewards. Hence, the adjacency matrix $\mathbf{A}$ is unknown a priori. In the following, we work with the matrix form of \Cref{eq:OverallRew}, defined at time $t$ as 
\begin{equation}
\label{eq:OverallRew-MatrixForm}
    \mathbf{y}_{t} = \mathbf{A} \mathbf{y}_{t} + \mathbf{F}\mathbf{z}_{t}.
\end{equation}
In \textbf{\Cref{fig:ExmpModel}}, we illustrate an exemplary network consisting of $N$ vertices and the underlying causal relations. Based on our problem formulation, the agent is able to observe both the exogenous input signal vector $\mathbf{z}_{t}$ and the endogenous output signal vector $\mathbf{y}_{t}$. As we see, there does not exist necessarily a causal relation between every pair of nodes.

%-----------------------------------------> Figure
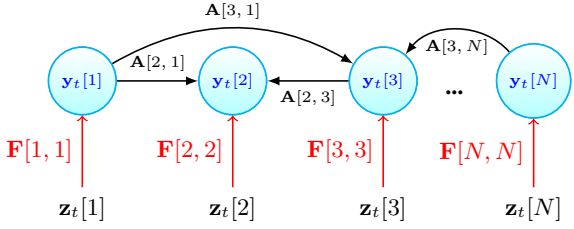
\begin{figure}[t]
\centering
\begin{tikzpicture}[-latex ,auto ,node distance =2 cm and 5cm ,on grid ,
semithick ,
state/.style ={ circle ,top color =white , bottom color = processblue!30 ,
draw,processblue , text=blue , minimum width =0.1 cm}]

\node[state](q1)  {\tiny $\mathbf{y}_{t}[1]$};
\node[state, right of=q1] (q2) {\tiny $\mathbf{y}_{t}[2]$};
\node[state, right of=q2] (q3) {\tiny $\mathbf{y}_{t}[3]$};
%\node[state, right of=q3] (q4) {$\mathbf{y}_{4t}$};
\node[state, right of=q3] (q4) {\tiny $\mathbf{y}_{t}[N]$};

% \draw (q1) edge[loop above] node {\tt $a_{12}$} (q1);
\draw (q1) edge[bend left = 0] node {\tt \tiny $\mathbf{A}{[2,1]}$} (q2);
%\draw (q2) edge[bend left= 8] node {\tt $a_{21}$} (q1);
\draw (q3) edge[bend left= 0] node {\tt \tiny $\mathbf{A}{[2,3]}$} (q2);
\draw (q1) edge[bend left= +30] node {\tt \tiny $\mathbf{A}{[3,1]}$} (q3);
\draw (q4) edge[bend left = -50] node {\tt \tiny $\mathbf{A}{[3,N]}$} (q3);

\node [right=0.7cm,text width=6cm,font=\footnotesize] at (q3){

\textbf{...}

};

\path (0,-1.7) node (x) {\small $\mathbf{z}_{t}[1]$}
(0,0) node[state](q1) {\tiny $\mathbf{y}_{t}[1]$};
\draw[->,red] (x) -- node[midway] {\small $\mathbf{F}[1,1]$} (q1);

\path (2,-1.7) node (x) {\small $\mathbf{z}_{t}[2]$}
(2,0) node[state](q2) {\tiny $\mathbf{y}_{t}[2]$};
\draw[->,red] (x) -- node[midway] {\small $\mathbf{F}[2,2]$} (q2);

\path (4,-1.7) node (x) {\small $\mathbf{z}_{t}[3]$}
(4,0) node[state](q3) {\tiny $\mathbf{y}_{t}[3]$};
\draw[->,red] (x) -- node[midway] {\small $\mathbf{F}[3,3]$} (q3);

%\path (6,-1.7) node (x) {$z_{t}[4]$}
%(6,0) node[state](q4) {$y_{t}[4]$};
%\draw[->,red] (x) -- node[midway] {$F[4,4]$} (q4);

\path (6,-1.7) node (x) {\small $\mathbf{z}_{t}[N]$}
(6,0) node[state](q4) {\tiny $\mathbf{y}_{t}[N]$};
\draw[->,red] (x) -- node[midway] {\small $\mathbf{F}[N,N]$} (q4);
 
\end{tikzpicture}
\caption{An exemplary illustration of a graph consisting of $N$ vertices and their causal relations. The black directed edges represent the causal relationships amongst the vertices.}
\label{fig:ExmpModel}
\end{figure}
%----------------------------------------->

By inserting (\ref{eq:ExoInput}) in (\ref{eq:OverallRew-MatrixForm}) and solving for $\mathbf{y}_{t}$ we obtain
\begin{equation}
\label{eq:OverallRew-Solved}
    \mathbf{y}_{t}=(\mathbf{I-A})^{-1}\mathbf{F}\operatorname{diag}(\mathbf{b}_{t}) \mathbf{x}_{t}.
\end{equation}
Finally, we define the \textit{payoff} received by the agent upon choosing the decision vector $\mathbf{x}_{t}$ as
\begin{equation}
\label{eq:Payoff}
    r(\mathbf{x}_{t}) = {\bf 1}^{\top}\mathbf{y}_{t} = {\bf 1}^{\top} (\mathbf{I-A})^{-1} \mathbf{F} \operatorname{diag}(\mathbf{b}_{t}) \mathbf{x}_{t}, 
\end{equation}
where ${\bf 1}$ is the $N$-dimensional vector of ones. Since the graph $\mathcal{G}$ is a DAG, it implies that with a proper indexing of the vertices, the adjacency matrix $\mathbf{A}$ is a strictly upper triangular matrix. This guarantees that the matrix $(\mathbf{I-A})$ is invertible. In our problem, since the agent directly observes the exogenous input, we assume that the effects of $\mathbf{F}$ on the exogenous inputs are already integrated in the instantaneous rewards. Therefore, to simplify the notation and without loss of generality, we assume that $\mathbf{F} = \mathbf{I}$ in the following.

Given a decision vector $\mathbf{x}_{t} \in \mathcal{X}$, the expected payoff at time $t$ is calculated as
\begin{equation}
\label{eq:ExpPayoff}
    \mu(\mathbf{x}_{t}) = \mathbb{E}\left [  r(\mathbf{X}) | \mathbf{X} = \mathbf{x}_{t} \right ],
\end{equation}
where the expectation is taken with respect to the randomness in the reward generating processes.

Ideally, the agent's goal is to maximize her total mean payoff over a time horizon $T$. Alternatively, the agent aims at minimizing the expected regret, defined as the difference between the expected accumulated payoff of an oracle that follows the optimal policy and that of the agent that follows the applied policy. Formally, the expected regret is defined as
\begin{equation}
\label{eq:ExpRegret}
    \mathcal{R}_{T}(\mathcal{X}) = T \mu(\mathbf{x}^{\ast}) - \sum_{t = 1}^{T} \mu(\mathbf{x}_{t}),
\end{equation}
where $\mathbf{x}^{*} = \underset{\mathbf{x} \in \mathcal{X}}{\text{argmax}} ~\mu(\mathbf{x})$ is the optimal decision vector, and $\mathbf{x}_{t}$ denotes the selected decision vector at time $t$ under the applied policy.

\begin{remark}
The definition of payoff in (\ref{eq:Payoff}) implies that we are dealing with a linear combinatorial semi-bandit problem with causally related rewards. In general, due to the randomness in selection of the decision vector $\mathbf{x}_{t}$, the consecutive overall reward vectors $\mathbf{y}_{t}$ become non-identically distributed. In the following section, we propose our algorithm that is able to deal with such variables. This is an improvement over the previous methods, such as \cite{chen2016combinatorial} and \cite{huyuk2019analysis}, that are not able to cope with our problem formulation, as they are specially designed to work with i.i.d. random variables.
\end{remark}
%
%-----------------------> Section: Decision-Making Strategy
\section{Decision-Making Strategy}
\label{sec:DMStrategy}
In this section, we present our decision-making strategy to solve the problem described in \Cref{sec:ProFor}. Our proposed policy consists of two learning components: (i) an online graph learning and (ii) an Upper Confidence Bound (UCB)-based reward learning. In the following, we describe each component separately and propose our algorithm, namely SEM-UCB.
%-----------------------> Subsection: Online Graph Learning
\subsection{Online Graph Learning}
\label{subsec:GraphLearning}
The payoff defined in (\ref{eq:Payoff}) implies that the knowledge of $\mathbf{A}$ is necessary to select decision vectors that result in higher accumulated payoffs. Therefore, the agent aims at learning the matrix $\mathbf{A}$ to improve her decision-making process. To this end, we propose an online graph learning framework that uses the collected feedback, i.e., the collected exogenous input and endogenous output vectors, to estimate the ground truth matrix $\mathbf{A}$. In the following, we formalize the online graph learning framework.

At each time $t$, we collect the feedback up to the current time in $\mathbf{Z}_{t} = [\mathbf{z}_{1} \hdots \mathbf{z}_{t}]$ and $\mathbf{Y}_{t} = [\mathbf{y}_{1} \hdots \mathbf{y}_{t}]$. Therefore,
\begin{equation}
\label{eq:OveralRew-2xMatrixForm}
    \mathbf{Y}_{t} = \mathbf{A} \mathbf{Y}_{t} + \mathbf{Z}_{t}.
\end{equation}
We assume that the right indexing of the vertices is known prior to estimating the ground truth adjacency matrix. We use the collected feedback $\mathbf{Y}_{t}$ and $\mathbf{Z}_{t}$ as the input to a parametric graph learning algorithm (\cite{giannakis2018topology}, \cite{dong2019learning}). More precisely, we use the following optimization problem to estimate the adjacency matrix at time $t$.
\begin{equation}
\label{eq:Optimization-A}
\begin{aligned}
\hat{\mathbf{A}}_{t} = \underset{\mathbf{A}}{\text{argmin}} \quad & \left \| \mathbf{Y}_{t} - \mathbf{A} \mathbf{Y}_{t} - \mathbf{Z}_{t} \right \|_{2}^{2} + \mathrm{g}(\mathbf{A})\\
\textrm{s.t.} \quad &\mathbf{A}[i,j] \geq 0, \quad \forall i, j \in [N], \\
&\mathbf{A}[i,j]=0, \quad \forall i \geq j,
\end{aligned}
\end{equation}
where $\left \| \cdot \right \|_{2}$ represents the $L^{2}\text{-norm}$ of matrices and $\mathrm{g}(\mathbf{A})$ is a regularization function that imposes sparsity over $\mathbf{A}$. In our numerical experiments, we work with different regularization functions to demonstrate the effectiveness of our proposed algorithm in different scenarios. As an example, we impose the sparsity property on the estimated matrix $\hat{\mathbf{A}}_{t}$ in (\ref{eq:Optimization-A}) by defining $\mathrm{g}(\mathbf{A}) = \lambda \left \| \mathbf{A} \right \|_{1}$, where $\left \| \cdot \right \|_{1}$ is the $L^{1}\text{-norm}$ of the matrices and $\lambda$ is the regularization parameter. Our choices of regularization function guarantee that the optimization problem (\ref{eq:Optimization-A}) is convex. 
% Thus, one can use a convex solver to solve this problem.}
%--> Subsection: SEM-UCB Algorithm
\subsection{SEM-UCB Algorithm}
\label{subsec:Alg}
We propose our decision-making policy in \textbf{Algorithm \ref{Alg:SEM-UCB}}. The key idea behind our algorithm is that it works with observations for each base arm, rather than the payoff observations for each super arm. As the same base arm can be observed while selecting different super arms, we can use the obtained information from selection of a super arm to improve our payoff estimation of other relevant super arms. This, combined with the fact that our algorithm simultaneously learns the causal relations, significantly improves the performance of our proposed algorithm and speed up the learning process.

For each base arm $i$, we define the empirical average of instantaneous rewards at time $t$ as
\begin{equation}
    \label{eq:AvgRew}
    \hat{\boldsymbol{\beta}}_{t}[i] = \frac{\sum_{\tau = 1}^{t} \mathbf{b}_{\tau}[i] \mathds{1}\left\{ \mathbf{x}_{\tau}[i] = 1\right\} }{\mathbf{m}_{t}[i]},
\end{equation}
where $\mathbf{m}_{t}[i]$ denotes the number of times that the base arm $i$ is observed up to time $t$. Formally,
\begin{equation}
    \label{eq:Counter-BaseArm}
    \mathbf{m}_{t}[i] = \sum_{\tau = 1}^{t} \mathds{1}\left\{ \mathbf{x}_{\tau}[i] = 1\right\}.
\end{equation}

The initialization phase of SEM-UCB algorithm follows a specific strategy to create a rich data that helps to learn the ground truth adjacency matrix. At each time $t$ during the first $N$ times of play, SEM-UCB picks the column $t$ of an upper-triangular \textbf{initialization matrix} $\mathbf{M} \in \{0, 1\}^{N \times N}$, where $\mathbf{M}$ is created as follows. All diagonal elements of $\mathbf{M}$ are equal to $1$. As for the column $i$, if $i \leq s$, we set all elements above diagonal to $1$. If $s + 1 \leq i \leq N$, we select $s-1$ elements above diagonal uniformly at random and set them to $1$. The remaining elements are set to $0$.

After the initialization period, our proposed algorithm takes two steps at each time $t$ to learn the causal relationships and the expected instantaneous rewards of the base arms. First, it uses the collected feedback $\mathbf{Y}_{t}$ and $\mathbf{Z}_{t}$ and solves the optimization problem (\ref{eq:Optimization-A}) to obtain the estimated adjacency matrix. It then uses the reward observations to calculate the UCB index $\mathbf{E}_{t}[i]$ for each base arm $i$, defined as
\begin{equation}
    \label{eq:UCB}
    \mathbf{E}_{t}[i] = \hat{\boldsymbol{\beta}}_{t}[i] + \sqrt{ \frac{(s+1) \text{ln} t}{\mathbf{m}_{t}[i]} }.
\end{equation}
Afterward, the algorithm selects a decision vector $\mathbf{x}_{t}$ using the current estimate of the adjacency matrix and the developed UCB indices of the base arms. Let $\mathbf{E}_{t} = [\mathbf{E}_{t}[1], \mathbf{E}_{t}[2], \dots, \mathbf{E}_{t}[N]]$. At time $t$, SEM-UCB selects $\mathbf{x}_{t}$ as
\begin{equation}
\label{eq:Optimization-x}
\begin{aligned}
\mathbf{x}_{t} = \underset{\mathbf{x} \in \mathcal{X}}{\text{argmax}} \quad & \mathbf{1}^{\top} (\mathbf{I} - \hat{\mathbf{A}}_{t-1})^{-1} \operatorname{diag}(\mathbf{E}_{t-1}) \mathbf{x} \\
\textrm{s.t.} \quad &\left \| \mathbf{x} \right \|_{0} \leq s.
\end{aligned}
\end{equation}
%
%-----------------------> Algorithm: SEM-UCB
\begin{algorithm}[t]
\caption{SEM-UCB: Structural Equation Model-Upper Confidence Bound}
\label{Alg:SEM-UCB}
\textbf{Input}: Parameter $s$, initialization matrix $\mathbf{M}$. \\
\vspace{-4mm}
\begin{algorithmic}[1] %[1] enables line numbers
\FOR{$t = 1, \dots, N$}

    \STATE Select column $t$ of the initialization matrix $\mathbf{M}$ as the decision vector $\mathbf{x}_{t}$.

    \STATE Observe $\mathbf{z}_{t}$ and $\mathbf{y}_{t}$.
    % \STATE Update $\mathbf{m}_{t}[i]$ and $\hat{\boldsymbol{\beta}}_{t}[i]$, $\forall i$~s.t.~$\mathbf{x}_{t}[i]=1$. Collect $\mathbf{Z}_{t}$ and $\mathbf{Y}_{t}$.
\ENDFOR
\FOR{$t = N+1, \dots, T$}
    \STATE Solve (\ref{eq:Optimization-A}) to obtain $\hat{\mathbf{A}}_{t-1}$.
    % the estimated adjacency matrix 

    \STATE Calculate $\mathbf{E}_{t-1}[i]$ using (\ref{eq:UCB}), $\forall i \in [N]$.
    
    \STATE Select decision vector $\mathbf{x}_{t}$ that solves (\ref{eq:Optimization-x}).
    
    \STATE Observe $\mathbf{z}_{t}$ and $\mathbf{y}_{t}$.
    
    % \STATE Update $\mathbf{m}_{t}[i]$ and $\hat{\boldsymbol{\beta}}_{t}[i]$, $\forall i$. Collect the feedback in $\mathbf{Z}_{t}$ and $\mathbf{Y}_{t}$.
\ENDFOR
\end{algorithmic}
\end{algorithm}
%-----------------------------------------
%
\begin{remark}
The initialization phase of our algorithm guarantees that all the base arms are pulled at least once and the matrix $\mathbf{M}$ is full rank. Consequently, the adjacency matrix $\mathbf{A}$ is uniquely identifiable from the collected feedback \cite{bazerque2013identifiability}.
\end{remark}
\begin{remark}
Let $\mathbf{c}^{\top} = \mathbf{1}^{\top} (\mathbf{I} - \hat{\mathbf{A}}_{t-1})^{-1} \operatorname{diag}(\mathbf{E}_{t-1})$. Since all the elements of both matrices $\mathbf{E}_{t-1}$ and $\hat{\mathbf{A}}_{t-1}$ are non-negative, we have $\mathbf{c}[i] > 0$, $\forall i \in [N]$. Thus, the optimization problem (\ref{eq:Optimization-x}) reduces to finding the $s$-biggest elements of $\mathbf{c}$. Therefore,
(\ref{eq:Optimization-x}) can be solved efficiently based on the choice of sorting algorithm used to order the elements of $\mathbf{c}$.

The computational complexity of the SEM-UCB algorithm varies depending on the solver that is used to learn the graph. For example, if we use OSQP solver \cite{osqp}, we achieve a computational complexity of order $\mathcal{O}(N^{4})$.
\end{remark}
%
%-----------------------------------------> Theoretical Analysis
\section{Theoretical Analysis}
\label{sec:RegAnalysis}
In this section, we prove an upper bound on the expected regret of SEM-UCB algorithm. We use the following definitions in our regret analysis. For any decision vector $\mathbf{x} \in \mathcal{X}$, let $\Delta(\mathbf{x}) = \mu(\mathbf{x}^{\ast}) - \mu(\mathbf{x})$. We define $\Delta_{\max} = \underset{\mathbf{x}: \mu(\mathbf{x}) < \mu(\mathbf{x}^{\ast})}{\max}~ \Delta(\mathbf{x})$ and $\Delta_{\min} = \underset{\mathbf{x}: \mu(\mathbf{x}) < \mu(\mathbf{x}^{\ast})}{\min}~ \Delta(\mathbf{x})$. Moreover, let $\mathbf{w}_{t}^{\top} = {\bf 1}^{\top}(\mathbf{I} - \hat{\mathbf{A}}_{t})^{-1}\text{diag}(\mathbf{x}_{t+1})$. We define $w_{\max} = \underset{t}{\max} ~\underset{i}{\max} ~\mathbf{w}_{t}[i]$.

The following theorem states an upper bound on the expected regret of SEM-UCB.
%-----------------------------------------> Theorem 1
\begin{theorem}
\label{thm:1}
The expected regret of SEM-UCB algorithm is upper bounded as

\begin{align} 
   &\mathcal{R}_{T}(\mathcal{X}) \leq {\Big[} \frac{4 w_{\max}^{2} s^{2} (s+1) N \ln{T} }{\Delta_{\min}^{2}} + N + \frac{\pi^{2}}{3} s^{p} N {\Big]} \Delta_{\max}.
\end{align}
% {\Big[} \frac{4 w_{\max}^{2} s^{2} (s+1) N \ln{T} }{\Delta_{\min}^{2}} + N + \frac{\pi^{2}}{3} s^{N} N {\Big]} \Delta_{\max}
\end{theorem}
%-----------------------------------------> Proof
\begin{proof}
See Section A of supplementary material.
\end{proof}
%-----------------------------------------> Section: Numerical Analysis
\section{Experimental Analysis}
\label{sec:NumAnalysis}
In this section, we present experimental results to provide more insight on the usefulness of learning the causal relations for improving the decision-making process. We evaluate the performance of our algorithm on synthetic and real-world datasets by comparing it to standard benchmark algorithms.

\paragraph{Benchmarks.~} We compare SEM-UCB with state-of-the-art combinatorial semi-bandit algorithms that do not learn the causal structure of the problem. Specifically, we compare our algorithm with the following policies: (i) CUCB \cite{chen2016combinatorial} calculates a UCB index for each base arm at each time $t$ and feeds them to an approximation oracle that outputs a super arm. (ii) DFL-CSR \cite{tang2017networked} develops a UCB index for each base arm and selects a super arm at each time $t$ based on a prior knowledge of a graph structure that shows the correlations among base arms. (iii) CTS \cite{huyuk2019analysis} employs Thompson sampling and uses an oracle to select a super arm at each time $t$. (iv) FTRL \cite{zimmert2019beating} selects a super arm at each time $t$ based on the method of Follow-the-Regularized-Leader. To be comparable, we apply these benchmarks on the vector of overall reward $\mathbf{y}_{t}$ at each time $t$. If a benchmark requires $\mathbf{y}_{t}$ to be in $[0,1]$, we feed the normalized version of $\mathbf{y}_{t}$ to the corresponding algorithm. Finally, in our experiments, we choose $s=6$, meaning that the algorithms can choose $6$ base arms at each time of play.
%-----------------------------------------> Synthetic Data
\subsection{Synthetic Dataset}
\label{subsec:SynData}
Our simulation setting is as follows. We first create a graph consisting of $N = 20$ nodes. The elements of the adjacency matrix $\mathbf{A}$ are drawn from a uniform distribution over $[0.4,0.7]$. The edge density of the ground truth adjacency matrix is $0.15$. At each time $t$, the vector of instantaneous rewards $\mathbf{b}_{t}$ is drawn from a multivariate normal distribution with the support in $[0,1]^{20}$ and a spherical covariance matrix. As demonstrated in \Cref{sec:ProFor}, we generate the vector of overall rewards according to the SEM in (\ref{eq:OverallRew}). We use $\mathrm{g}(\mathbf{A}) = \lambda \left \| \mathbf{A} \right \|_{1}$ as the regularization function in (\ref{eq:Optimization-A}) when estimating the adjacency matrix $\mathbf{A}$. The regularization parameter $\lambda$ is tuned by grid search over $[0.0001,1000]$. We evaluate the estimated adjacency matrix at each time $t$ by using the mean squared error defined as $\text{MSE} = \frac{1}{N^{2}} \left\| \mathbf{A} - \hat{\mathbf{A}} \right\|_{\text{F}}^{2}$, where $\left \| \cdot \right \|_{\text{F}}$ denotes the Frobenius norm.

\paragraph{Comparison with the benchmarks.~} We run the algorithms using the aforementioned synthetic data with $T = 4000$. In \textbf{\Cref{fig:Reg-Synthetic}}, we depict the trend of time-averaged expected regret for each policy. As we see, SEM-UCB surpasses all other policies. This is due to the fact that SEM-UCB learns the network's topology and hence, it has a better knowledge of the causal relationships in the graph structure, unlike other policies that do not estimate the graph structure. As we see, the time-averaged expected regret of SEM-UCB tends to zero. This matches with our theoretical results in \Cref{sec:RegAnalysis}. Note that, the benchmark policies exhibit a suboptimal regret performance as they have to deal with non-identically distributed random variables $\mathbf{y}_{t}$.
%-----------------------------------------> Figure
\begin{figure}[t]
    \centerline{
    \includegraphics*[width=0.48\textwidth]{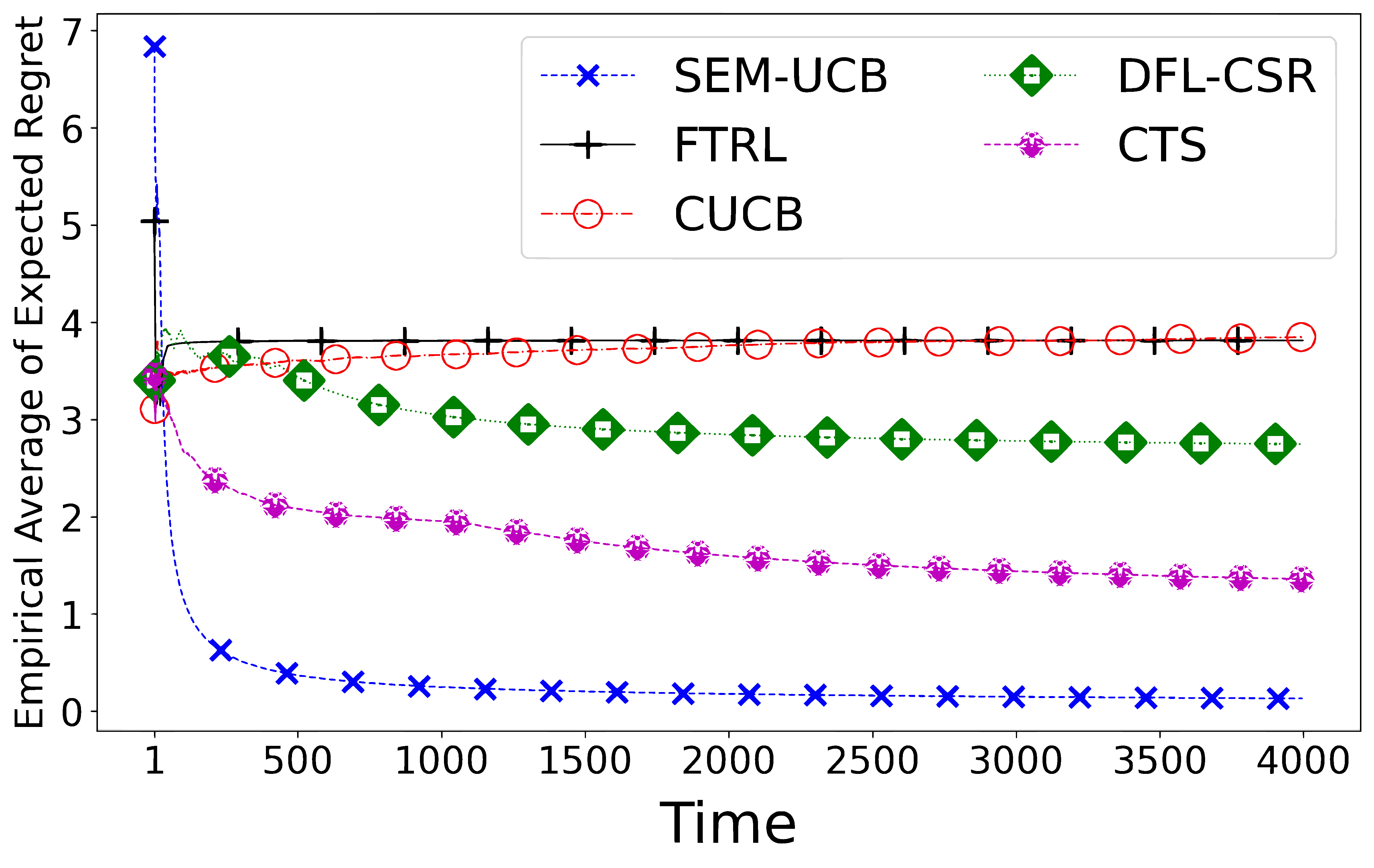}
    }
    \caption{Time-averaged expected regret of different policies.}
    \label{fig:Reg-Synthetic}
\end{figure}
%
%-----------------------------------------

%-----------------------------------------> Real-World data
\subsection{Covid-19 Dataset}
\label{subsec:RealData}
We evaluate our proposed algorithm on the Covid-19 outbreak dataset of daily new infected cases during the pandemic in different regions within Italy.\footnote{https://github.com/pcm-dpc/COVID-19} The dataset fits in our framework as the daily new cases in each region results from the causal spread of Covid-19 among the regions in a country \cite{NEURIPS2020_205e7357} and the region-specific characteristics \cite{guaitoli2021covid}. As the regions differ in their regional characteristics, such as socio-economic and geographical characteristics, each region has a specific exposure risk of Covid-19 infection.  To be consistent with our terminology in \Cref{sec:ProFor}, at each time (day) $t$, we use the \textit{overall reward} $\mathbf{y}_{t}[i]$ to refer to the \textit{overall daily new cases} in region $i$ and use the \textit{instantaneous reward} $\mathbf{b}_{t}[i]$ to refer to the \textit{region-specific daily new cases} in region $i$. Naturally, the overall daily new cases includes the region-specific daily new cases of Covid-19 infection.

Governments around the world strive to track the spread of Covid-19 and find the regions that are contributing the most to the total number of daily new cases in the country \cite{bridgwater2021identifying}. By the end of this experiment, we address this critical problem and highlight that our algorithm is capable of finding the optimal candidate regions for political interventions in order to contain the spread of a contagious disease such as Covid-19.

\paragraph{Data preparation.~} We focus on the recorded daily new cases from $10$ August to $15$ October, $2020$, for $N = 21$ regions within Italy. The Covid-19 dataset only provides us with the overall daily new cases of each region. Hence, in order to apply our algorithm, we need to infer the distribution of region-specific daily new cases for each region. In the following, we describe this process and further pre-processing of the Covid-19 dataset.

According to \cite{bull2021italian}, for the time period from $18$ May to $3$ June, $2020$, all places for work and leisure activities were opened and travelling within regions was permitted while travelling between regions was forbidden. Consequently, during this period, there are no causal effects on the overall daily new cases of each region from other regions. In addition, according to google mobility data \cite{nouvellet2021reduction}, during $4$ weeks prior to $18$ May the mobility was increasing within the regions while travel ban between the regions was still imposed. Hence, we use this expanded period to estimate the underlying distributions of the region-specific daily new cases using a kernel density estimation. Finally, considering that the daily recorded data noticeably fluctuates, a $7$-day moving average was applied to the signals. 
% by averaging the values of that day, the day
% before, and the next day.

We create the region-specific daily new cases for each region by sampling from the estimated distributions. Below, we present the results of applying our algorithm on the pre-processed Covid-19 dataset. Since the data only contains the reported overall daily new cases for a limited time period, care should be exercised in interpreting the results. However, by providing more relevant data, our proposed framework helps towards more accurate detection of the regions that contribute the most to the development of Covid-19.

\paragraph{Learning the structural dependencies.~} Our algorithm learns the ground truth adjacency matrix $\mathbf{A}$ using (\ref{eq:Optimization-A}). As for the choice of regularization function in (\ref{eq:Optimization-A}), we employ Directed Total Variation (DTV) which is a novel application of the Graph Directed Variation (GDV) function \cite{sardellitti2017graph}. DTV regularization function is defined as
\begin{equation}
    \mathrm{g}(\mathbf{A}) = \lambda \sum_{i,j = 1, \dots, N} \mathbf{A}[i,j] \sum_{k= 1, \dots, t} \left [ \mathbf{Y}[i,k] - \mathbf{Y}[j,k] \right]^{+},
\end{equation}
\begin{equation}
    \left[ y \right]^{+} = \text{max} \left\{ y, 0 \right \}.
\end{equation}
The regularization function addresses the smoothness of the entire observations $\mathbf{Y}$ over the underlying directed graph. To be more realistic, since the causal spread of the disease might create cycles, we additionally include cyclic graphs in the search space of the optimization problem (\ref{eq:Optimization-A}).
%-----------------------------------------> Figure
\begin{figure}[t]
    \centerline{
    \includegraphics*[width=0.48\textwidth]{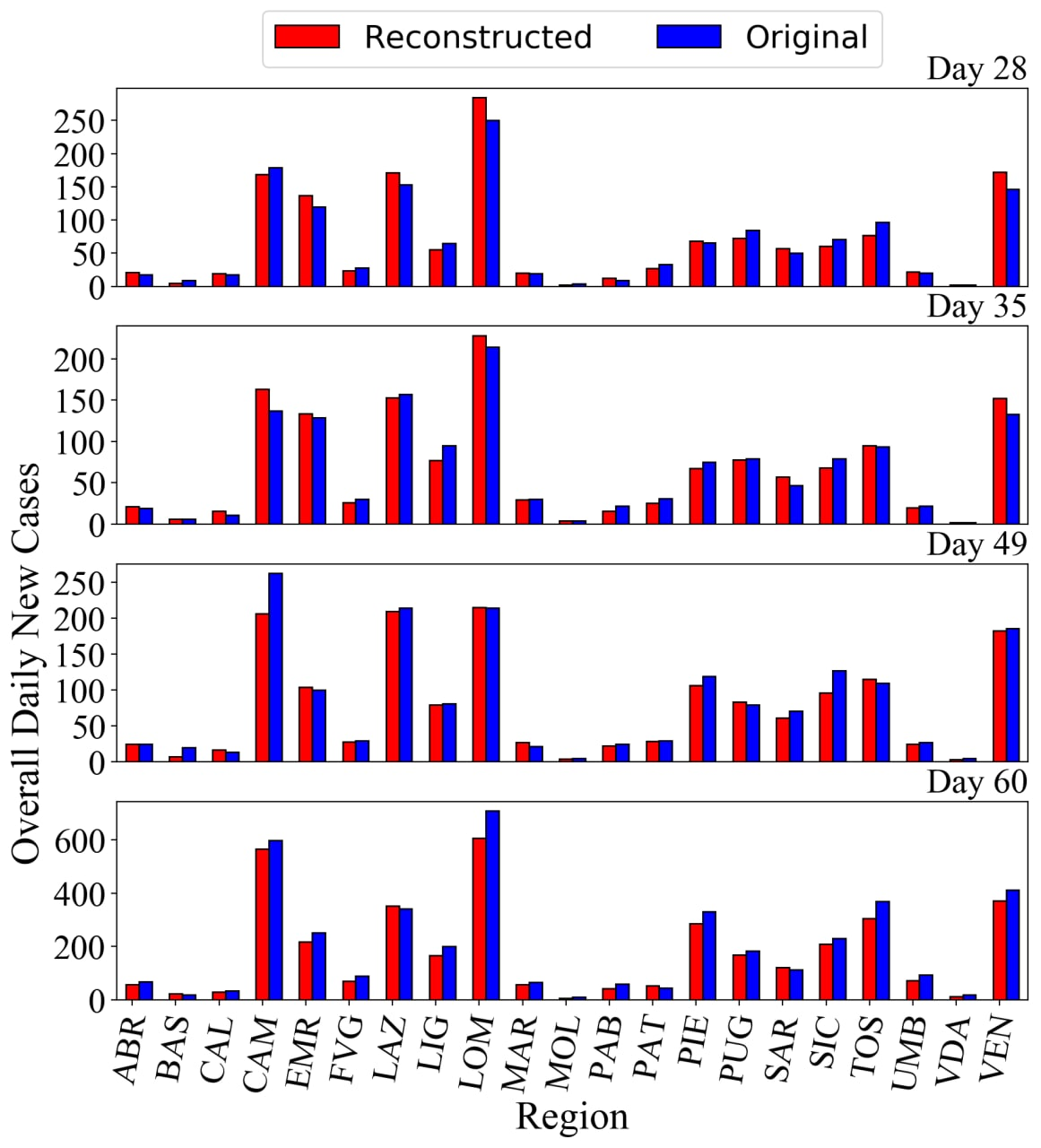}
    }
    \caption{Original overall daily new cases and the corresponding predicted values for different days in the validation set.}
    \label{fig:Error}
\end{figure}
%
%-----------------------------------------

We perform cross-validation technique to tune the regularization parameter $\lambda$. As mentioned before, we work on a limited time period with $T = 66$ days. Thus, we split the data into train and validation sets in 10:1 ratio. More specifically, we split the data into $6$ subsets of $11$ consecutive days. In each subset, one day is chosen uniformly at random to be included in the validation set while the remaining $10$ days are added to the train set. We calculate the prediction error at each time $t$ by
\begin{equation}
\label{eq:Error}
    \textit{Error}(t) = \frac{1}{N K(t)} \sum_{i \in \mathcal{K}(t)} \left\| \mathbf{y}_{i} - \hat{\mathbf{y}}_{i} \right\|_{1},
\end{equation}
where $\mathcal{K}(t)$ is the validation set at time $t$ with cardinality $K(t) = |\mathcal{K}(t)|$ and $\mathbf{y}_{i}$ and $\hat{\mathbf{y}}_{i}$ are the validation data and the corresponding predicted value using the estimated graph for day $i$, respectively. \textbf{\Cref{fig:Error}} compares the ground truth overall daily new cases and the predicted overall daily new cases using the estimated graph on $4$ different days of the Covid-19 outbreak in our validation data. Due to space limitation, we use abbreviations for region names. Table $1$ in Section B.1 of supplementary material lists the abbreviations together with the original names of the regions. We observe that our proposed framework is capable to estimate the data for each region efficiently, that helps the agent to improve its decision-making process in a real-world scenario.

\paragraph{Learning regions with the highest
contribution.~} 
In \textbf{\Cref{fig:SuperArmSelection}}, we show the decision-making process of the agent over time by following the SEM-UCB policy. Dark rectangles represent the $6$ selected regions at each day (time). Based on our framework, we represent the selected regions by our algorithm as those with biggest contributions to the development of Covid-19 during the time interval considered in our experiment. More specifically, we find the regions of Lombardia, Emilia-Romagna, Lazio, Veneto, Piemonte, and Liguria as the ones that contribute the most to the spread of Covid-19 during that period in Italy.

We emphasize that, due to the causal effects among the regions, contribution of each region to the spread of covid-19 differs from its overall daily cases of infection. Thus, the set of regions with the highest contribution does not necessarily equal to the set of regions with the highest total number of daily cases. This is a key aspect of our problem formulation that is addressed by SEM-UCB in \textbf{\Cref{fig:SuperArmSelection}}. We elaborate more on this fact in Section B.3 of supplementary material.
%-----------------------------------------> Figure
\begin{figure}[t]
    \centerline{\includegraphics*[width=.48\textwidth]{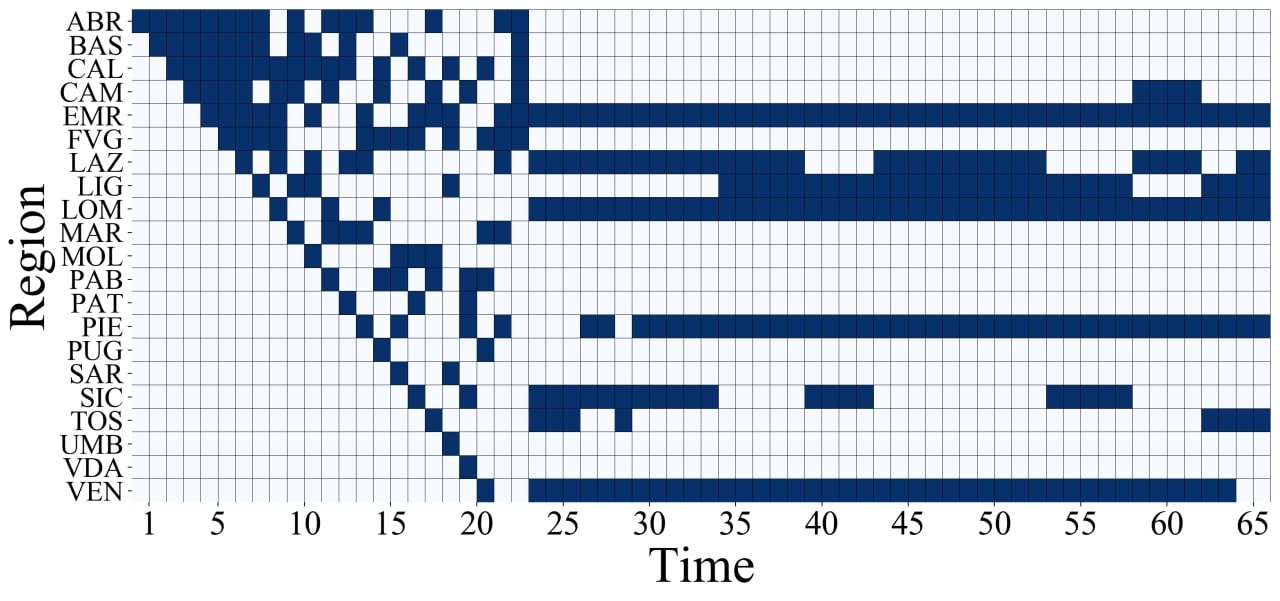}}
    \caption{Selected regions on each day.}
    \label{fig:SuperArmSelection}
\end{figure}
%
%-----------------------------------------

%-----------------------------------------> Section: Conclusion
\section{Conclusion}
\label{sec:Conclusion}
In this paper, we developed a combinatorial semi-bandit framework with causally related rewards, where we modelled the causal relations by a directed graph in a structural equation model. We developed a decision-making policy, namely SEM-UCB, that learns the structural dependencies to improve the decision-making process. We proved that SEM-UCB achieves a sublinear regret bound in time. Our framework is applicable in a number of contexts such as network data analysis of biological networks or financial markets. We applied our method to analyze the development of Covid-19. The experiments showed that SEM-UCB outperforms several state-of-the-art combinatorial semi-bandit algorithms. Future research directions would be to extend the current framework to deal with piece-wise stationary environments where the causal graph and/or the expected instantaneous rewards of the base arms undergo abrupt changes over time.

%------------------------------%
\section*{Acknowledgements}This work was partially funded by the Deutsche Forschungsgemeinschaft (DFG, German Research Foundation) under Germany’s Excellence Strategy – EXC number 2064/1 – Project number 390727645, and by Grant 01IS20051 from the German Federal Ministry of Education and Research (BMBF). We are grateful to Sergio Barbarossa and Sofien Dhouib for fruitful discussions and comments.
\newpage
%-----------------------------------------> References
%% The file named.bst is a bibliography style file for BibTeX 0.99c
\bibliographystyle{named}
\bibliography{thebibliography}

%--------------
\appendix
\section{PROOF OF THEOREM 1}
\label{sec:ProofThm1}
%
%-------------------------------> Notations
\subsection{NOTATIONS}
\label{subsec:Notation}
Before proceeding to the proof, in the following we introduce some important notations together with their definitions.

We define the \textit{index set} of a decision vector $\mathbf{x} \in \mathcal{X}$ by $\mathcal{I}(\mathbf{x}) = \left\{ i ~|~ \mathbf{x}[i] \neq 0, \forall i \in [N] \right\}$. The confidence bound of base arm $i$ at time $t$ is defined as $\mathbf{C}_{t}[i] = \sqrt{\frac{(s+1) \ln{t}}{\mathbf{m}_{t}[i]}}$. At each time $t$, we collect the empirical average of instantaneous rewards $\hat{\boldsymbol{\beta}}_{t}[i]$ and the calculated confidence bounds $\mathbf{C}_{t}[i]$ of all base arms $i \in [N]$ in vectors $\hat{\boldsymbol{\beta}}_{t}$ and $\mathbf{C}_{t}$, respectively. We have $\mathbf{E}_{t} = \hat{\boldsymbol{\beta}}_{t} + \mathbf{C}_{t}$. For ease of presentation, in the sequel, we use the following equivalence ${\bf 1}^{\top} (\mathbf{I} - \hat{\mathbf{A}}_{t-1})^{-1} \text{diag}(\mathbf{E}_{t-1}) \mathbf{x}_{t} = {\bf 1}^{\top} (\mathbf{I} - \hat{\mathbf{A}}_{t-1})^{-1} \text{diag}(\mathbf{x}_{t}) \mathbf{E}_{t-1}$.
At each time $t$, we define the \textit{selection index} for a decision vector $\mathbf{x} \in \mathcal{X}$ as $I_{t}(\mathbf{x}) = \mathbf{1}^{\top}({\bf I} - \hat{\mathbf{A}}_{t-1})^{-1} \textup{diag}(\mathbf{x}) \mathbf{E}_{t-1}$. To simplify the notation, sometimes we drop the time index $t$ in $\mathbf{m}_{t}[i]$ and use $\mathbf{m}[i]$ to denote the number of times that the base arm $i$ has been observed up to the current time instance.

For any $\mathbf{x} \in \mathcal{X}$, we use the counter $\mathcal{T}_{\mathbf{x}}(t)$ to represent the total number of times the decision vector $\mathbf{x}$ is selected up to time $t$. Finally, for each base arm $i \in [N]$, we define a counter $\mathscr{T}_{i}(t)$ which is updated as follows. At each time $t$ after the initialization phase that a suboptimal decision vector $\mathbf{x}_{t}$ is selected, we have at least one base arm $i \in [N]$ such that $i = \underset{i \in \mathcal{I}(\mathbf{x}_{t})}{\textup{argmin}} ~\mathbf{m}_{t}[i]$. In this case, if the base arm $i$ is unique, we increment $\mathscr{T}_{i}(t)$ by 1. If there are more than one such base arm, we break the tie and select one of them arbitrarily to increment its corresponding counter.
%-------------------------------
%-------------------------------> Auxiliary
\subsection{Auxiliary Results }
\label{subsec:aux}
We use the following lemma in the proof of Theorem 1.
\paragraph{Lemma 1.}{\textup{(\cite{Azuma67:WSC})}}
\label{lem:Hoeffding}
Let $z_{1}, z_{2}, \dots, z_{m}$ be random variables and $z_{i} \in [0,1]$, $\forall i$. Moreover, $\mathbb{E}[z_{t} | z_{1}, \dots, z_{t-1}] = \alpha$, for all $t = 1, \dots, m$. Then, for all $D \geq 0$,
\begin{equation}
     \mathbb{P} {\Bigg[} {\Big |} \sum_{i = 1}^{m} z_{i} - m \alpha {\Big |} \geq D {\Bigg]} \leq e^{- \frac{2 D^{2}}{m} }.
\end{equation}
%
%\end{lemma}
%
%-------------------------------
%-------------------------------> Proof of Thm 1
\subsection{PROOF}
\label{subsec:Proof1}
We start by rewriting the expected regret as
\begin{align} \nonumber
    \mathcal{R}_{T}(\mathcal{X}) 
    &= T \mu(\mathbf{x}^{\ast}) - \sum_{t = 1}^{T} \mu(\mathbf{x}_{t}) \\
    &= \sum_{\mathbf{x}:\mu(\mathbf{x}) < \mu(\mathbf{x}^{\ast})}^{} \Delta(\mathbf{x}) \mathbb{E} [\mathcal{T}_{\mathbf{x}}(T)].
\end{align}
Based on the definition of the counters $\mathscr{T}_{i}(t)$ for the base arms $i \in [N]$, at each time $t$ that a suboptimal decision vector is selected, only one of such counters is incremented by $1$. Thus, we have \cite{Gai12:CNO}
\begin{align}
    \mathbb{E}\left[ \sum_{\mathbf{x}:\mu(\mathbf{x}) < \mu(\mathbf{x}^{\ast})}^{} \mathcal{T}_{\mathbf{x}}(t) \right]
    = \mathbb{E} \left[\sum_{i = 1}^{N} \mathscr{T}_{i}(t)\right ],
\end{align}
which implies that
\begin{align}
    \sum_{\mathbf{x}:\mu(\mathbf{x}) < \mu(\mathbf{x}^{\ast})}^{} \mathbb{E} \left[\mathcal{T}_{\mathbf{x}}(t)\right]
    = \sum_{i = 1}^{N} \mathbb{E} \left[\mathscr{T}_{i}(t)\right].
\end{align}
Therefore, we observe that
\begin{align} \nonumber
    \mathcal{R}_{T}(\mathcal{X})
    &= \sum_{\mathbf{x}:\mu(\mathbf{x}) < \mu(\mathbf{x}^{\ast})}^{} \Delta(\mathbf{x}) \mathbb{E} [\mathcal{T}_{\mathbf{x}}(T)] \\
    &\stackrel{(\ast)}{\leq} \Delta_{\max} \sum_{i = 1}^{N} \mathbb{E} [\mathscr{T}_{i}(T)],
\end{align}
where $(\ast)$ follows from the definition of $\Delta_{\max}$.

Let $\mathbbm{I}_{i}(t)$ denote the indicator function which is equal to $1$ if $\mathscr{T}_{i}(t)$ is increased by $1$ at time $t$, and is $0$ otherwise. Therefore,
\begin{align}
    \mathscr{T}_{i}(T) = \sum_{t = N+1}^{T} \mathbbm{1}\left\{ \mathbbm{I}_{i}(t) = 1 \right\}.
\end{align}
If $\mathbbm{I}_{i}(t) = 1$, it means that a suboptimal decision vector $\mathbf{x}_{t}$ is selected at time $t$. In this case, $\mathbf{m}_{t}[i] = \min \left\{ \mathbf{m}_{t}[j] | j \in \mathcal{I}(\mathbf{x}_{t}) \right\}$. Let $l = \left \lceil {\frac{4 (s+1) \ln{T}}{ (\frac{\Delta_{\min}}{s w_{\max}})^{2} }} \right \rceil$. Then,
\begin{align} \nonumber
   &\mathscr{T}_{i}(T) 
   = \sum_{t = N+1}^{T} \mathbbm{1}\left\{ \mathbbm{I}_{i}(t) = 1 \right\} \\ \nonumber
   &\leq l + 
    \sum_{t = N+1}^{T} \mathbbm{1}\left\{ \mathbbm{I}_{i}(t) = 1 ~\&~ \mathscr{T}_{i}(t-1) \geq l \right\} \\ \nonumber
   &\leq l + \sum_{t = N+1}^{T} \mathbbm{1}\left\{ I_{t}(\mathbf{x}^{\ast}) \leq I_{t}(\mathbf{x}_{t}) ~\&~ \mathscr{T}_{i}(t-1) \geq l\right \} \\ \nonumber 
   &= l + \sum_{t = N+1}^{T} \mathbbm{1}\{ {\bf 1}^{\top} (\mathbf{I} - \hat{\mathbf{A}}_{t - 1})^{-1} \text{diag}(\mathbf{x}^{\ast}) \mathbf{E}_{t-1} \\  \nonumber
   &\hspace{8mm} \leq {\bf 1}^{\top} (\mathbf{I} - \hat{\mathbf{A}}_{t - 1})^{-1} \text{diag}(\mathbf{x}_{t}) \mathbf{E}_{t-1} ~\&~ \mathscr{T}_{i}(t-1) \geq l \} \\ \nonumber
   &= l + \sum_{t = N}^{T} \mathbbm{1}\{ {\bf 1}^{\top} (\mathbf{I} - \hat{\mathbf{A}}_{t})^{-1} \text{diag}(\mathbf{x}^{\ast}) \mathbf{E}_{t} \\
   &\hspace{8mm}  \leq {\bf 1}^{\top} (\mathbf{I} - \hat{\mathbf{A}}_{t})^{-1} \text{diag}(\mathbf{x}_{t+1}) \mathbf{E}_{t} ~\&~ \mathscr{T}_{i}(t) \geq l \}.
\end{align}
Based on the definition of $\mathscr{T}_{i}(t)$, we have $\mathscr{T}_{i}(t) \leq \mathbf{m}_{t}[i]$, $\forall i \in [N]$. Therefore, when $\mathscr{T}_{i}(t) \geq l$, the following holds \cite{Gai12:CNO}.
\begin{align}
    l \leq \mathscr{T}_{i}(t) \leq \mathbf{m}_{t}[j], ~~~~ \forall j \in \mathcal{I}(\mathbf{x}_{t+1}).
\end{align}
Let $\mathbf{v}_{t+1}^{\top} = {\bf 1}^{\top} (\mathbf{I} - \hat{\mathbf{A}}_{t})^{-1} \text{diag}(\mathbf{x}^{\ast})$ and $\mathbf{u}_{t+1}^{\top} = {\bf 1}^{\top} (\mathbf{I} - \hat{\mathbf{A}}_{t})^{-1} \text{diag}(\mathbf{x}_{t+1})$. We order the elements in sets $\mathcal{I}(\mathbf{x}^{\ast})$ and $\mathcal{I}(\mathbf{x}_{t+1})$ arbitrarily. In the following, our results are independent of the way we order these sets. Let $v_{k}$, $k = 1, \dots, |\mathcal{I}(\mathbf{x}^{\ast})| \leq s$, represent the $k$th element in $\mathcal{I}(\mathbf{x}^{\ast})$ and $u_{k}$, $k = 1, \dots, |\mathcal{I}(\mathbf{x}_{t+1})| \leq s$, represent the $k$th element in $\mathcal{I}(\mathbf{x}_{t+1})$. Hence, we have
\begin{align} \nonumber
\label{eq:event}
   &\mathscr{T}_{i}(T)
   \leq l + \sum_{t = N}^{T} \mathbbm{1} {\Bigg\{} \min_{0 < \mathbf{m}[v_{1}], \dots, \mathbf{m}[v_{|\mathcal{I}(\mathbf{x}^{\ast})|}] \leq t} \\ \nonumber
   &\hspace{35mm}\sum_{j=1}^{|\mathcal{I}(\mathbf{x}^{\ast})|} \mathbf{v}_{t+1}^{\top}[v_{j}] (\hat{\boldsymbol{\beta}}_{t}[v_{j}] + \mathbf{C}_{t}[v_{j}]) \leq \\ \nonumber
   &\max_{l \leq \mathbf{m}[u_{1}], \dots,
   \mathbf{m}[u_{|\mathcal{I}(\mathbf{x}_{t+1})|}] \leq t} 
   \sum_{j=1}^{|\mathcal{I}(\mathbf{x}_{t+1})|} \mathbf{u}_{t+1}^{\top}[u_{j}] (\hat{\boldsymbol{\beta}}_{t}[u_{j}] + \mathbf{C}_{t}[u_{j}]) {\Bigg \}} \\ \nonumber
   &\leq l + \sum_{t=1}^{\infty} \sum_{m_{v_{1}}=1}^{t} \dots \sum_{m_{v_{|\mathcal{I}(\mathbf{x}^{\ast})|}}=1}^{t} \sum_{m_{u_{1}}=l}^{t} \dots \sum_{m_{u_{|\mathcal{I}(\mathbf{x}_{t+1})|}}=l}^{t} \\ \nonumber
   &\hspace{10mm}\mathbbm{1} {\Bigg \{} 
%   \underbrace{
   \sum_{j=1}^{|\mathcal{I}(\mathbf{x}^{\ast})|} \mathbf{v}_{t+1}^{\top}[v_{j}] (\hat{\boldsymbol{\beta}}_{t}[v_{j}] + \mathbf{C}_{t}[v_{j}]) \\ 
   &\hspace{25mm}\leq \sum_{j=1}^{|\mathcal{I}(\mathbf{x}_{t+1})|} \mathbf{u}_{t+1}^{\top}[u_{j}] (\hat{\boldsymbol{\beta}}_{t}[u_{j}] + \mathbf{C}_{t}[u_{j}])
%   }_{\mathcal{P}} 
   {\Bigg \}}. 
\end{align}
We define the Event $\mathcal{P}$ as
\begin{align} \nonumber
\label{eq:eventp-separated}
    \sum_{j=1}^{|\mathcal{I}(\mathbf{x}^{\ast})|} \mathbf{v}_{t+1}^{\top}&[v_{j}] (\hat{\boldsymbol{\beta}}_{t}[v_{j}] + \mathbf{C}_{t}[v_{j}]) \\
   &\leq \sum_{j=1}^{|\mathcal{I}(\mathbf{x}_{t+1})|} \mathbf{u}_{t+1}^{\top}[u_{j}] (\hat{\boldsymbol{\beta}}_{t}[u_{j}] + \mathbf{C}_{t}[u_{j}]).
\end{align}
If the Event $\mathcal{P}$ in (\ref{eq:eventp-separated}) is true, it implies that at least one of the following events must be true.
\begin{align} \label{eq:part1} \nonumber
    {\bf 1}^{\top} (\mathbf{I} - \hat{\mathbf{A}}_{t})^{-1} \text{diag}&(\mathbf{x}^{\ast}) (\hat{\boldsymbol{\beta}}_{t} + \mathbf{C}_{t})  \\
    &\hspace{-4mm}\leq {\bf 1}^{\top} (\mathbf{I} - \mathbf{A})^{-1} \text{diag}(\mathbf{x}^{\ast})\boldsymbol{\beta}, \\ \nonumber \label{eq:part2}
    {\bf 1}^{\top} (\mathbf{I} - \hat{\mathbf{A}}_{t})^{-1} \text{diag}&(\mathbf{x}_{t+1}) (\hat{\boldsymbol{\beta}}_{t} - \mathbf{C}_{t}) \\
    &\hspace{-4mm}\geq {\bf 1}^{\top} (\mathbf{I} - \mathbf{A})^{-1} \text{diag}(\mathbf{x}_{t+1})\boldsymbol{\beta}, \\ \nonumber
    {\bf 1}^{\top} (\mathbf{I} - \mathbf{A})^{-1} \text{diag}(\mathbf{x}^{\ast})\boldsymbol{\beta}
    &<
    {\bf 1}^{\top} (\mathbf{I} - \mathbf{A})^{-1} \text{diag}(\mathbf{x}_{t+1})\boldsymbol{\beta} \\ \label{eq:part3}
    &\hspace{-8mm}+ 2 {\bf 1}^{\top} (\mathbf{I} - \hat{\mathbf{A}}_{t})^{-1} \text{diag}(\mathbf{x}_{t+1}) \mathbf{C}_{t}.
    \end{align}
First, we consider (\ref{eq:part1}). Based on our problem formulation and proposed solution, we know that matrices $\mathbf{A}$ and $\hat{\mathbf{A}}_{t}$ are nilpotent with index $N$. Thus, $\mathbf{A}^{N} = \mathbf{0}_{N \times N}$ and $\hat{\mathbf{A}}_{t}^{N} = \mathbf{0}_{N \times N}$. Hence, we can write the Taylor's series of $(\mathbf{I} - \mathbf{A})^{-1}$ and $(\mathbf{I} - \hat{\mathbf{A}}_{t})^{-1}$ as
\begin{align}
\label{eq:nilpotentA}
    \hspace{-7mm}(\mathbf{I} - \mathbf{A})^{-1} = \mathbf{I} + \mathbf{A} + \mathbf{A}^{2} + \dots + \mathbf{A}^{N-1},
\end{align}
and
\begin{align}
\label{eq:nilpotentAhat}
    (\mathbf{I} - \hat{\mathbf{A}}_{t})^{-1} = \mathbf{I} + \hat{\mathbf{A}}_{t} + \hat{\mathbf{A}}_{t}^{2} + \dots + \hat{\mathbf{A}}_{t}^{N-1},
\end{align}
respectively. Substituting (\ref{eq:nilpotentA}) and (\ref{eq:nilpotentAhat}) in (\ref{eq:part1}) results in
\begin{align} \nonumber
    {\bf 1}^{\top} (\mathbf{I} + \hat{\mathbf{A}}_{t} &+ \hat{\mathbf{A}}_{t}^{2} + \dots + \hat{\mathbf{A}}_{t}^{N-1}) \text{diag}(\mathbf{x}^{\ast}) (\hat{\boldsymbol{\beta}}_{t} + \mathbf{C}_{t}) \\
    &\leq {\bf 1}^{\top} (\mathbf{I} + \mathbf{A} + \mathbf{A}^{2} + \dots + \mathbf{A}^{N-1}) \text{diag}(\mathbf{x}^{\ast}) \boldsymbol{\beta}.
\end{align}
For $j = 1, \dots N$, we find the upper bound for
\begin{align} %\nonumber
\label{eq:part1-j}
    \mathbb{P} \left[ {\bf 1}^{\top} \hat{\mathbf{A}}_{t}^{j-1} \text{diag}(\mathbf{x}^{\ast}) (\hat{\boldsymbol{\beta}}_{t} + \mathbf{C}_{t}) \leq
    {\bf 1}^{\top} \mathbf{A}^{j-1} \text{diag}(\mathbf{x}^{\ast}) \boldsymbol{\beta} \right].
\end{align}
We consider the following Event $\mathcal{E}$.
%
% \begin{align} \nonumber
    % {\bf 1}^{\top} \hat{\mathbf{A}}_{t+1}^{j-1} \text{diag}(\mathbf{x}^{\ast}) (\hat{\boldsymbol{\beta}}(t) + \mathbf{C}(t)) 
    % - {\bf 1}^{\top} \hat{\mathbf{A}}_{t+1}^{j-1} \text{diag}(\mathbf{x}^{\ast})\boldsymbol{\beta}
    % &\leq {\bf 1}^{\top} A^{j-1} \text{diag}(\mathbf{x}^{\ast})\boldsymbol{\beta} 
    % - {\bf 1}^{\top} \hat{\mathbf{A}}_{t+1}^{j-1} \text{diag}(\mathbf{x}^{\ast})\boldsymbol{\beta} \\
    % {\bf 1}^{\top} \hat{\mathbf{A}}_{t+1}^{j-1} \text{diag}(\mathbf{x}^{\ast}) (\hat{\boldsymbol{\beta}}(t) + \mathbf{C}(t) -\boldsymbol{\beta}) &\leq {\bf 1}^{\top} (A^{j-1} - \hat{\mathbf{A}}_{t+1}^{j-1}) \text{diag}(\mathbf{x}^{\ast})\boldsymbol{\beta}
% \end{align}
%
%
\begin{align} \nonumber
    {\bf 1}^{\top} \hat{\mathbf{A}}_{t}^{j-1} \text{diag}&(\mathbf{x}^{\ast}) (\hat{\boldsymbol{\beta}}_{t} + \mathbf{C}_{t}) 
    + {\bf 1}^{\top} \hat{\mathbf{A}}_{t}^{j-1} \text{diag}(\mathbf{x}^{\ast}) \boldsymbol{\beta} \\
    &\leq {\bf 1}^{\top} \hat{\mathbf{A}}_{t}^{j-1} \text{diag}(\mathbf{x}^{\ast}) \boldsymbol{\beta} + {\bf 1}^{\top} \mathbf{A}^{j-1} \text{diag}(\mathbf{x}^{\ast}) \boldsymbol{\beta}.
\end{align}
If $\mathcal{E}$ is true, then at least one of the following must hold.
\begin{align} %\nonumber
    &\underbrace{{\bf 1}^{\top} \hat{\mathbf{A}}_{t}^{j-1} \text{diag}(\mathbf{x}^{\ast}) (\hat{\boldsymbol{\beta}}_{t} + \mathbf{C}_{t})
    \leq {\bf 1}^{\top} \hat{\mathbf{A}}_{t}^{j-1} \text{diag}(\mathbf{x}^{\ast})\boldsymbol{\beta}}_{\mathcal{I}}, \\
    &\underbrace{{\bf 1}^{\top} \hat{\mathbf{A}}_{t}^{j-1} \text{diag}(\mathbf{x}^{\ast}) \boldsymbol{\beta}
    \leq {\bf 1}^{\top} \mathbf{A}^{j-1} \text{diag}(\mathbf{x}^{\ast})\boldsymbol{\beta}}_{\mathcal{II}}.
\end{align}
Therefore, we have
\begin{align} %\nonumber
\label{eq:eventE}
    \mathbb{P} \left[ \mathcal{E} \right] \leq \mathbb{P} \left[ \mathcal{I} \right] + \mathbb{P} \left[ \mathcal{II} \right].
\end{align}
Let $\mathbf{y}_{t}^{\top} = {\bf 1}^{\top} \hat{\mathbf{A}}_{t}^{j-1} \text{diag}(\mathbf{x}^{\ast})$. If Event $\mathcal{I}$ is true, then at least one of the following must hold.
\begin{align} %\nonumber
    \mathbf{y}_{t}^{\top}[v_{1}] (\hat{\boldsymbol{\beta}}_{t}[v_{1}] + \mathbf{C}_{t}[v_{1}])
    \leq \mathbf{y}_{t}^{\top}[v_{1}] &\boldsymbol{\beta}[v_{1}], \\
    \mathbf{y}_{t}^{\top}[v_{2}] (\hat{\boldsymbol{\beta}}_{t}[v_{2}] + \mathbf{C}_{t}[v_{2}])
    \leq \mathbf{y}_{t}^{\top}[v_{2}] &\boldsymbol{\beta}[v_{2}], \\ \nonumber
    &\vdots \\ \nonumber
    \mathbf{y}_{t}^{\top}[v_{|\mathcal{I}(\mathbf{x}^{\ast})|}] (\hat{\boldsymbol{\beta}}_{t}[v_{|\mathcal{I}(\mathbf{x}^{\ast})|}] + \mathbf{C}_{t}[v_{|\mathcal{I}(\mathbf{x}^{\ast})|}]) \\
    &\hspace{-20mm}\leq \mathbf{y}_{t}^{\top}[v_{|\mathcal{I}(\mathbf{x}^{\ast})|}] \boldsymbol{\beta}[v_{|\mathcal{I}(\mathbf{x}^{\ast})|}].
\end{align}
For $k = 1, \dots, |\mathcal{I}(\mathbf{x}^{\ast})|$, we have
\begin{align} \nonumber
    \mathbb{P}{\Big[} \mathbf{y}_{t}^{\top}[v_{k}] &(\hat{\boldsymbol{\beta}}_{t}[v_{k}] + \mathbf{C}_{t}[v_{k}])
    \leq \mathbf{y}_{t}^{\top}[v_{k}] \boldsymbol{\beta}[v_{k}]{\Big]} \\ \nonumber
    &\stackrel{(a)}{=} \mathbb{P}{\Big[} \mathbf{m}_{t}[v_{k}] (\hat{\boldsymbol{\beta}}_{t}[v_{k}] + \mathbf{C}_{t}[v_{k}]) \leq \mathbf{m}_{t}[v_{k}] \boldsymbol{\beta}[v_{k}]{\Big]}  \\ \nonumber
    &\stackrel{(b)}{\leq} e^{-(2/\mathbf{m}_{t}[v_{k}]) \mathbf{m}_{t}[v_{k}]^{2} \mathbf{C}_{t}[v_{k}]^{2}} \\ \nonumber
    &\stackrel{(c)}{=} e^{-2 (s+1) \ln t} \\
    &= t^{-2 (s+1)},
\end{align}
where $(a)$ holds since $\mathbf{y}_{t}^{\top}[v_{k}] \geq 0$, $\forall k$, $(b)$ follows from Lemma 1, and $(c)$ results from the definition of $\mathbf{C}_{t}$. Hence, for Event $\mathcal{I}$, we conclude that
\begin{align}
    \mathbb{P}\left[\mathcal{I}\right] \leq |\mathcal{I}(\mathbf{x}^{\ast})| t^{-2 (s+1)} \leq s t^{-2 (s+1)}.
\end{align}
Now, we consider Event $\mathcal{II}$. Based on Theorem 1 in \cite{Bazerque13:IOS}, we know that we can identify the adjacency matrix $A$ uniquely by $N$ samples gathered during the initialization period of our proposed algorithm. This means that with probability $1$, after the time point $\theta = N < \infty$, $\hat{\mathbf{A}}_{t} = \mathbf{A}$ holds for all $t > \theta$. Therefore, for $t > N$, Event $\mathcal{II}$ holds with probability $1$.

Combining the aforementioned results with (\ref{eq:eventE}), we find the upper bound for (\ref{eq:part1-j}) as
\begin{align} \nonumber
\label{eq:part1-jj}
    \mathbb{P} {\Big[} {\bf 1}^{\top} \hat{\mathbf{A}}_{t}^{j-1} \text{diag}(\mathbf{x}^{\ast}) (\hat{\boldsymbol{\beta}}_{t} + \mathbf{C}_{t})  
    \leq
    {\bf 1}^{\top} \mathbf{A}^{j-1} &\text{diag}(\mathbf{x}^{\ast})\boldsymbol{\beta} {\Big]}  \\
    &\leq st^{-2(s+1)},
\end{align}
for each $j = 1, \dots, N$. Since $\hat{\mathbf{A}}_{t} = \mathbf{A}$, $\forall t > N$ and the length of the largest path in the graph is $p$, we can rewrite (\ref{eq:nilpotentA}) and (\ref{eq:nilpotentAhat}) as \cite{duncan2004powers}
\begin{align}
\label{eq:nilpotentA-p}
    (\mathbf{I} - \mathbf{A})^{-1} = \mathbf{I} + \mathbf{A} + \mathbf{A}^{2} + \dots + \mathbf{A}^{p},
\end{align}
and
\begin{align}
\label{eq:nilpotentAhat-p}
    (\mathbf{I} - \hat{\mathbf{A}}_{t})^{-1} = \mathbf{I} + \hat{\mathbf{A}}_{t} + \hat{\mathbf{A}}_{t}^{2} + \dots + \hat{\mathbf{A}}_{t}^{p},
\end{align}
respectively. Therefore, by using (\ref{eq:nilpotentA-p}) and (\ref{eq:nilpotentAhat-p}) in place of (\ref{eq:nilpotentA}) and (\ref{eq:nilpotentAhat}), and based on (\ref{eq:part1-jj}), the following holds for (\ref{eq:part1}).
\begin{align} \nonumber
\label{eq:part1-done}
    \mathbb{P} {\Big[} {\bf 1}^{\top} (\mathbf{I} - &\hat{\mathbf{A}}_{t})^{-1} \text{diag}(\mathbf{x}^{\ast}) (\hat{\boldsymbol{\beta}}_{t} + \mathbf{C}_{t}) \\
    &\leq {\bf 1}^{\top} (\mathbf{I} - \mathbf{A})^{-1} \text{diag}(\mathbf{x}^{\ast})\boldsymbol{\beta} {\Big]}  \leq s^{p}t^{-2p(s+1)}.
\end{align}
For (\ref{eq:part2}), we have similar results as follows.
\begin{align} \nonumber
\label{eq:part2-done}
    \mathbb{P} {\Big[}  {\bf 1}^{\top} (\mathbf{I} - &\hat{\mathbf{A}}_{t})^{-1} \text{diag}(\mathbf{x}_{t+1}) (\hat{\boldsymbol{\beta}}_{t} - \mathbf{C}_{t}) \\
    &\geq {\bf 1}^{\top} (\mathbf{I} - \mathbf{A})^{-1} \text{diag}(\mathbf{x}_{t+1})\boldsymbol{\beta} {\Big]}  \leq s^{p}t^{-2p(s+1)}.
\end{align}
Finally, we consider (\ref{eq:part3}). We have
\begin{align} \nonumber
\label{eq:part3-done}
    &{\bf 1}^{\top} (\mathbf{I} - \mathbf{A})^{-1} \text{diag}(\mathbf{x}^{\ast})\boldsymbol{\beta}
    - {\bf 1}^{\top} (\mathbf{I} - \mathbf{A})^{-1} \text{diag}(\mathbf{x}_{t+1})\boldsymbol{\beta} \\ \nonumber
    &\hspace{30mm}- 2 {\bf 1}^{\top} (\mathbf{I} - \hat{\mathbf{A}}_{t})^{-1} \text{diag}(\mathbf{x}_{t+1}) \mathbf{C}_{t} \\ \nonumber
    &\stackrel{(a)}{=} {\bf 1}^{\top} (\mathbf{I} - \mathbf{A})^{-1} \text{diag}(\mathbf{x}^{\ast})\boldsymbol{\beta}
    - {\bf 1}^{\top} (\mathbf{I} - \mathbf{A})^{-1} \text{diag}(\mathbf{x}_{t+1})\boldsymbol{\beta}  \\ \nonumber
    &\hspace{30mm}- 2 \sum_{j: j \in \mathcal{I}(\mathbf{x}_{t+1})}^{} \mathbf{w}_{t}^{\top}[j] \mathbf{C}_{t}[j] \\ \nonumber
    &\stackrel{(b)}{=} {\bf 1}^{\top} (\mathbf{I} - \mathbf{A})^{-1} \text{diag}(\mathbf{x}^{\ast})\boldsymbol{\beta}
    - {\bf 1}^{\top} (\mathbf{I} - \mathbf{A})^{-1} \text{diag}(\mathbf{x}_{t+1})\boldsymbol{\beta} \\ \nonumber
    &\hspace{30mm}- 2 \sum_{j: j \in \mathcal{I}(\mathbf{x}_{t+1})}^{} \mathbf{w}_{t}^{\top}[j] \sqrt{\frac{(s+1) \ln{t}}{\mathbf{m}_{t}[j]}} \\ \nonumber
    &\stackrel{(c)}{\geq} {\bf 1}^{\top} (\mathbf{I} - \mathbf{A})^{-1} \text{diag}(\mathbf{x}^{\ast})\boldsymbol{\beta}
    - {\bf 1}^{\top} (\mathbf{I} - \mathbf{A})^{-1} \text{diag}(\mathbf{x}_{t+1})\boldsymbol{\beta} \\ \nonumber
    &\hspace{30mm}- 2 s w_{\max} \sqrt{\frac{(s+1) \ln{T}}{l}} \\ \nonumber
    &\stackrel{(d)}{\geq} {\bf 1}^{\top} (\mathbf{I} - \mathbf{A})^{-1} \text{diag}(\mathbf{x}^{\ast})\boldsymbol{\beta} - {\bf 1}^{\top} (\mathbf{I} - \mathbf{A})^{-1} \text{diag}(\mathbf{x}_{t+1})\boldsymbol{\beta} \\ \nonumber
    &\hspace{4mm}- \Delta_{\min}
    \\ \nonumber
    &\stackrel{(e)}{\geq} {\bf 1}^{\top} (\mathbf{I} - \mathbf{A})^{-1} \text{diag}(\mathbf{x}^{\ast})\boldsymbol{\beta} - {\bf 1}^{\top} (\mathbf{I} - \mathbf{A})^{-1} \text{diag}(\mathbf{x}_{t+1})\boldsymbol{\beta} \\ 
    &\hspace{4mm}- \Delta(\mathbf{x}_{t+1}) = 0,
\end{align}
where in $(a)$ and $(c)$ we used the definition of $\mathbf{w}_{t}^{\top}$ and $w_{\max}$, respectively. Moreover, in $(b)$ and $(d)$, we substituted the value for $\mathbf{C}_{t}[j]$ and $l$, respectively. $(e)$ follows from the definition of $\Delta_{\min}$. Hence, we conclude that (\ref{eq:part3}) never happens.

By using (\ref{eq:part1-done}), (\ref{eq:part2-done}), and (\ref{eq:part3-done}), we achieve the following.
% Summarizing all the above results yields
%
\begin{align} \nonumber
\mathbb{E}&[\mathscr{T}_{i}(T)] 
\leq \left \lceil {\frac{4 (s+1) \ln{T}}{ (\frac{\Delta_{\min}}{s w_{\max}})^{2} }} \right \rceil \\ \nonumber
&+ \sum_{t=1}^{\infty} {\Bigg[} \sum_{m_{v_{1}}=1}^{t} \dots \sum_{m_{v_{s}}=1}^{t} \sum_{m_{u_{1}}=l}^{t} \dots \sum_{m_{u_{s}}=l}^{t} 2 s^{p} t^{-2 p (s+1)} {\Bigg]} \\ \nonumber
&\leq \frac{4 w_{\max}^{2} s^{2} (s+1) \ln{T}}{\Delta_{\min}^{2}} + 1 + s^{p} \sum_{t=1}^{\infty} 2t^{-2} \\
&\leq \frac{4 w_{\max}^{2} s^{2} (s+1) \ln{T}}{\Delta_{\min}^{2}} + 1 + \frac{\pi^{2}}{3} s^{p}.
\end{align}
Therefore, the expected regret is upper bounded as
\begin{align} \nonumber
    \mathcal{R}_{T}(\mathcal{X})
    &\leq \Delta_{\max} \sum_{i = 1}^{N} \mathbb{E} [\mathscr{T}_{i}(T)] \\ \nonumber
    &\leq \sum_{i=1}^{N} {\Big[} \frac{4 w_{\max}^{2} s^{2} (s+1) \ln{T}}{\Delta_{\min}^{2}} + 1 + \frac{\pi^{2}}{3} s^{p} {\Big]} \Delta_{\max} \\
    &\leq {\Big[} \frac{4 w_{\max}^{2} s^{2} (s+1) N \ln{T} }{\Delta_{\min}^{2}} + N + \frac{\pi^{2}}{3} s^{p} N {\Big]} \Delta_{\max}.
\end{align}
\begin{flushright}
$\blacksquare$
\end{flushright}
%------------------------------->
% \newpage
% ~\newpage
%-------------------------------> Table of abbreviations
\section{Additional Information and Experiments Regarding Covid-19 Dataset}
\label{sec:AddExp}
\subsection{Abbreviations of Regions in Italy}
\Cref{tab:AbbList} lists the abbreviations together with the original names of the $21$ regions in Italy that we study in our numerical experiments.
%------------------------------->
\bigskip
\noindent
\begin{table}[b!]
\centering
\caption{List of regions in Italy and the corresponding abbreviations.}
\label{tab:AbbList}
\begin{tabular}{|l|*{4}{c|}} \toprule
% \backslashbox{}
{Abbreviation}
& \makebox[7em]{Region Name}
\\ \midrule
ABR & Abruzzo \\\hline
BAS & Basilicata \\\hline
 CAL & Calabria \\\hline
 CAM & Campania \\\hline
 EMR & Emilia-Romagna \\\hline
 FVG & Friuli Venezia Giulia \\\hline
 LAZ & Lazio \\\hline
 LIG & Liguria \\\hline
 LOM  & Lombardia \\\hline
 MAR  & Marche \\\hline
 MOL & Molise \\\hline
 PAB  & Provincia Autonoma di Bolzano \\\hline
 PAT  & Provincia Autonoma di Trento \\\hline
 PIE  & Piemonte \\\hline
 PUG  & Puglia \\\hline
 SAR  & Sardegna / Sardigna \\\hline
 SIC  & Siciliana \\\hline
 TOS  & Toscana \\\hline
 UMB  & Umbria \\\hline
 VDA  & Valle d'Aosta / Vallée d'Aoste \\\hline
 VEN  & Veneto \\ \bottomrule
\end{tabular}
\end{table}
\bigskip
%------------------------------->
\subsection{Overall Daily New Cases of Covid-19 Infection}
Italy has been severely affected by the COVID-19 pandemic. In April $2020$, the country had the highest death toll in Europe. From the beginning of the pandemic, with the goal of containing the outbreak, the Italian government has put in place an increasing number of restrictions. \Cref{fig1} depicts the overall daily new cases of covid-19 of the $21$ regions in Italy for the considered time interval in our numerical experiments. Due to space limitation, we use abbreviations for region names. \Cref{tab:AbbList} lists the abbreviations together with the original names of the regions.
%------------------------------->
\subsection{Additional Experiments}
As the governments try to contain the spread of Covid-19, they usually adopt restrictive measures such as quarantine over the regions that are showing the most number of overall daily new infections. As a result, they destructively ignore the effects of causal spread of the virus, meaning that they only focus on the overall daily new cases of regions without their causal effects on other regions. Therefore, we refer to this method of finding the best political interventions as the \textit{naive approach}. Our goal is to show the superiority of our proposed algorithm over this \textit{naive approach}. Similar to the experiments in our paper, we run SEM-UCB to find the $6$ regions that are contributing the most to the total number of daily new cases in Italy. 
\Cref{fig2} compares the performance of our algorithm with that of the naive approach. The diagram shows the ratio of the amount of contributions of the selected regions by the algorithms over the total number of daily new infections in the country for each day. As expected, after the initialization phase, SEM-UCB learns the underlying graph that influences the data. Consequently, it performs better with respect to the naive approach due to the fact that it takes the effects of causalities into account. We note that, due to such causal effects, it might be the case that a region with a lower number of overall daily cases contributes more than other regions with higher number of overall daily cases. This diagram provides the evidence that our framework can be highly effective in real-world applications such as analysis of the spread of Covid-19.

%------------------------------->
\begin{figure}[t]
   \centerline{\includegraphics*[width=.43\textwidth]{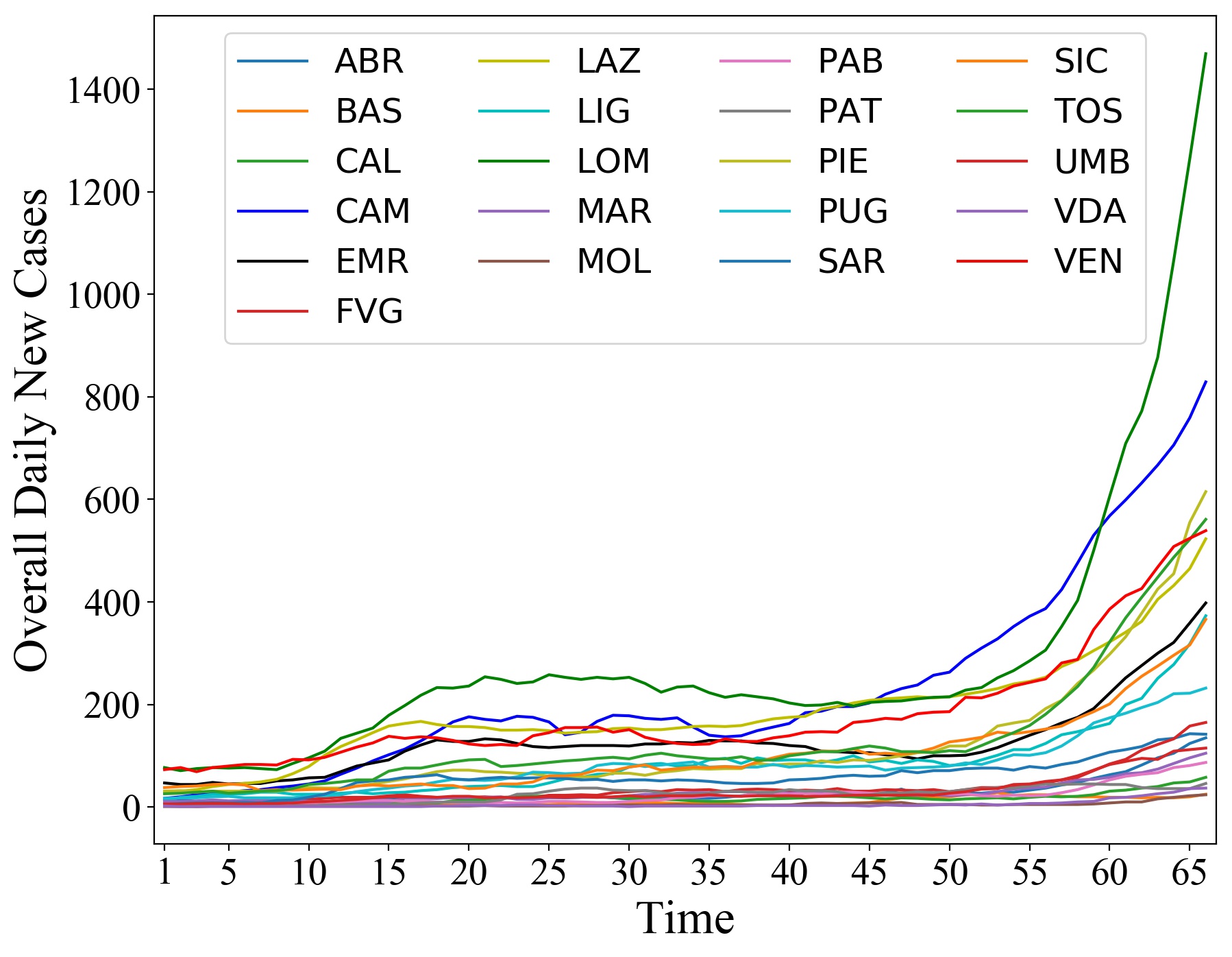}}
   \caption{Overall daily new cases of Covid-19 for different regions in Italy during the study period.}
\label{fig1}
\end{figure}
%-------------------------------

%------------------------------->
\begin{figure}[t]
   \centerline{\includegraphics*[width=.43\textwidth]{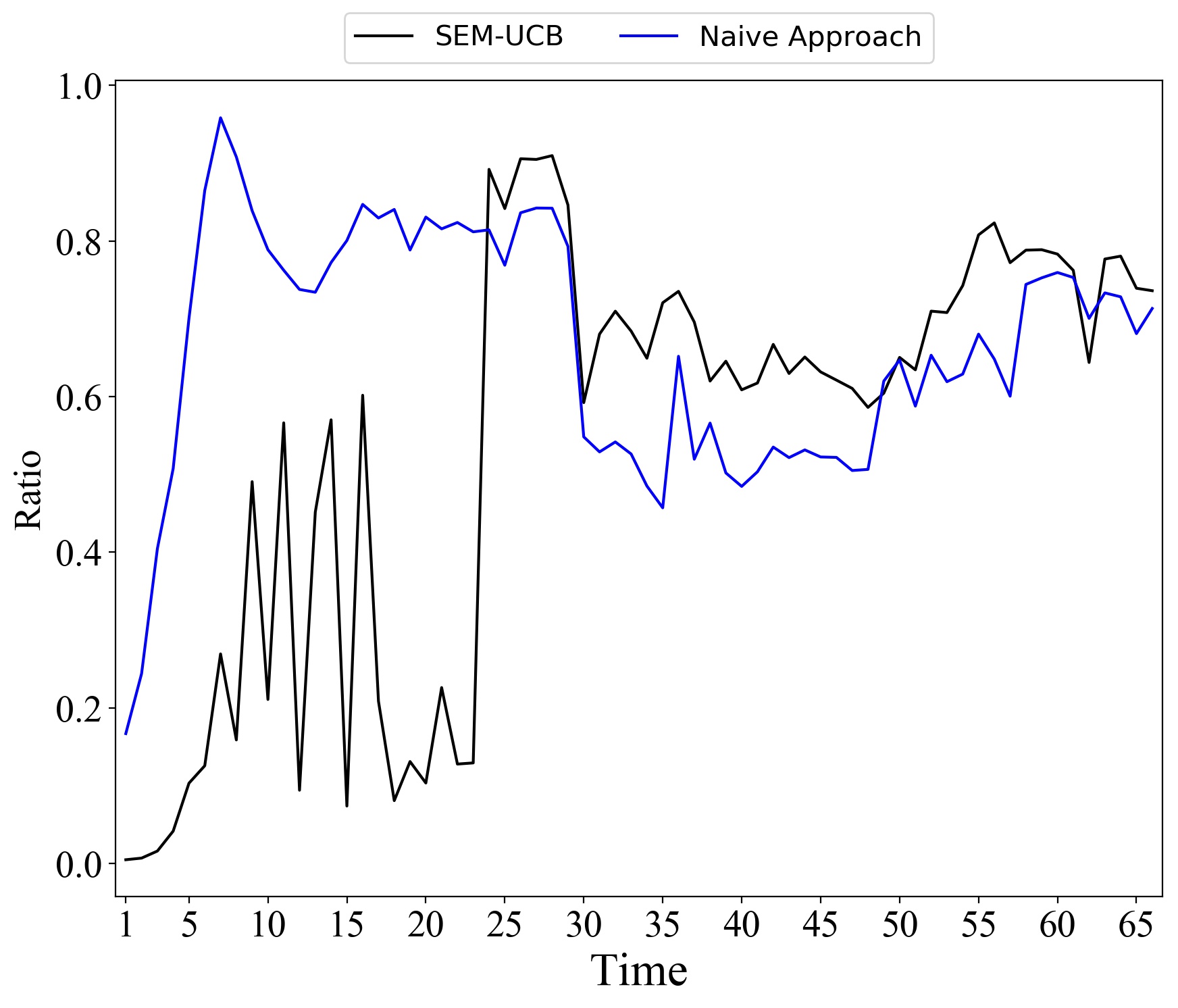}}
   \caption{The ratio of the amount of contributions of the selected regions by SEM-UCB and the naive approach over the total number of daily new infections in the country for each day.}
\label{fig2}
\end{figure}

%-------------------------------
% \newpage
%-------------------------------> References
%\bibliographystyle{named}
%\bibliography{ref-supp}
%-------------------------------

\end{document}